\PassOptionsToPackage{dvipsnames}{xcolor}
\PassOptionsToPackage{hypertexnames=false}{hyperref}
\documentclass{article}
\usepackage{enumitem}

\usepackage{etoolbox}
\newtoggle{colt}
\newcommand{\colt}[1]{\iftoggle{colt}{#1}{}}
\newcommand{\arxiv}[1]{\iftoggle{colt}{}{#1}}

\colt{
\crefname{appendix}{Appendix}{Appendices}
\Crefname{appendix}{Appendix}{Appendices}
}

\colt{
\title[Steering diffusion models with quadratic rewards]{Steering diffusion models with quadratic rewards: \\a fine-grained analysis}
\usepackage{times}
}
\arxiv{ 
\title{Steering diffusion models with quadratic rewards: \\a fine-grained analysis}

}
\colt{
\coltauthor{%
 \Name{Author Name1} \Email{abc@sample.com}\\
 \addr Address 1
 \AND
 \Name{Author Name2} \Email{xyz@sample.com}\\
 \addr Address 2%
}
}

\usepackage[utf8]{inputenc} % allow utf-8 input

\usepackage[T1]{fontenc}    % use 8-bit T1 fonts
\usepackage{hyperref}       % hyperlinks

\usepackage{url}            % simple URL typesetting
\usepackage{booktabs}       % professional-quality tables
\usepackage{amsfonts}       % blackboard math symbols
\usepackage{latexsym,amssymb,amscd,amsopn,amsmath}
\usepackage{bbm}
\usepackage{mathrsfs}
\usepackage{nicefrac}       % compact symbols for 1/2, etc.
\usepackage{microtype}      % microtypography
\usepackage{xcolor}         % colors
\usepackage{algorithm}
\usepackage{verbatim}
\usepackage[noend]{algpseudocode}

\usepackage{xspace}
\usepackage{mathtools}
\usepackage{svg}

\usepackage{makecell}
\usepackage{aliascnt}
\arxiv{
\usepackage[nameinlink,capitalize]{cleveref}
\usepackage{amsthm}
\usepackage{geometry}[1in]
}
\usepackage{thmtools}

\definecolor{dgreen}{rgb}{0,0.5,0}
\hypersetup{
  colorlinks=true,
  linkcolor=blue,
  filecolor=blue,
  citecolor = dgreen,      
  urlcolor=cyan,
}

\arxiv{
\theoremstyle{plain}
\newtheorem{theorem}{Theorem}
\newtheorem{lemma}[theorem]{Lemma}
\newtheorem{corollary}[theorem]{Corollary}

\newtheorem{assumption}[theorem]{Assumption}
\theoremstyle{definition}
\newtheorem{definition}[theorem]{Definition}

\newtheorem{remark}[theorem]{Remark}

\crefname{assumption}{Assumption}{Assumptions}
}
\numberwithin{theorem}{section}
\numberwithin{lemma}{section}
\numberwithin{definition}{section}

\newenvironment{namedproof}[1]{\paragraph{Proof of #1.}\hspace{-1em}}{\hfill$\blacksquare$\vspace{1em}}

\newcommand{\nc}{\newcommand}
\nc{\DMO}{\DeclareMathOperator}
\newcount\Comments  % 0 suppresses notes to selves in text
\DeclareMathOperator*{\argmax}{arg\,max}

\nc{\Moracle}{\MM^{\mathsf{oracle}}}
\nc{\tilMoracle}{\til{\MM}^{\mathsf{oracle}}}
\nc{\SimulateReduction}{\texttt{SimulateReduction}\xspace}
\nc{\TestReduction}{\texttt{DistinguishReduction}\xspace}
\nc{\SimulateSampling}{\texttt{SimulateSampling}\xspace}
\nc{\SimulateRegression}{\texttt{SimulateRegression}\xspace}
\nc{\Osample}{{\MO_{\mathsf{samp}}}}
\nc{\Oregress}{{\MO_{\mathsf{regress}}}}
\nc{\Oregressp}{{\MO'_{\mathsf{regress}}}}
\nc{\Obandits}{{\MO_{\mathsf{bandits}}}}
\nc{\dom}{\mathsf{dom}}
\nc{\SF}{\mathscr{F}}
\nc{\Fchernoff}{\MF^{\mathsf{chernoff}}}

\DMO{\prox}{prox}
\DMO{\Span}{span}
\DMO{\UCB}{UCB}
\DMO{\LCB}{LCB}
\nc{\expl}[2]{\ME^{#1}_{#2}}
\nc{\tilmdp}[1]{\til \MM({#1})}
\nc{\barpdp}[2]{\ol \MP_{#1}({#2})}
\nc{\barmdp}[1]{\ol \MM({#1})}
\nc{\hatmdp}[1]{\wh \MM({#1})}
\nc{\rem}[2]{\MR_{#1}({#2})}
\nc{\Pigen}{\Pi^{\rm gen}}
\nc{\Pidet}{\Pi^{\rm det}}
\nc{\PiZ}{\Pi_{\SZ}^{\rm markov}}
\nc{\und}[3]{\MU_{{#1}}^{{#2}}({#3})}
\nc{\zlow}[2]{\MZ_{{#1}}^\lowv({#2})}
\nc{\dg}{\dagger}
\nc{\bB}{\mathbf{B}}
\nc{\unif}{\mu_{\rm unif}}
\nc{\indsig}[2]{\mathcal{I}_{#1}({#2})}
\nc{\total}{{\rm fin}}
\nc{\early}{{\rm pre}}
\nc{\zsink}{z_{\rm sink}}
\nc{\lowv}{{\rm low}}
\nc{\oo}[1]{\texttt{o}({#1})}
\nc{\posnrm}[1]{\left[ {#1} \right]_+}
\nc{\negnrm}[1]{\left[ {#1} \right]_-}
\nc{\tvnrm}[1]{\left\| {#1} \right\|_1}
\nc{\absval}[1]{\left| {#1} \right|}
\nc{\normalize}[1]{\mathfrak{n}\left({#1}\right)}

\nc{\SZ}{\textsf{Z}}
\nc{\SO}{\textsf{O}}
\nc{\suff}[2]{{\rm suff}_{#1}({#2})}
\nc{\UPhi}{\mathscr{U}_{X,H}}
\nc{\UPhis}{\til{\mathscr{U}}_{X,H,\MF}}
\nc{\SV}{\mathscr{V}}
\nc{\Phiset}{\Phi_{X,H}}
\nc{\Phisets}{\til{\Phi}_{X,H,\MF}}
\nc{\Lyu}{{\mathtt{Lyu}}}
\nc{\wAlg}{{\widetilde \Alg}}

\nc{\ApproxMDP}{\texttt{ConstructMDP}\xspace}
\nc{\mainalg}{\texttt{BaSeCAMP}\xspace} % BArycentric SpannEr policy Cover with Approximate MdP
\nc{\bspanner}{\texttt{BarySpannerPolicy}\xspace}

\nc{\gamvec}{\gamma}
\nc{\til}{\widetilde}
\nc{\td}{\tilde}
\nc{\wh}{\widehat}
\nc{\old}[1]{\ifnum\Comments=1 {\color{brown}  [COPIED: #1]}\fi}
\definecolor{darkgreen}{rgb}{0.0, 0.5, 0.0}
\nc{\noah}[1]{\ifnum\Comments=1 {\color{darkgreen} [ng: #1]}\fi}
\nc{\dhruv}[1]{\ifnum\Comments=1 {\color{purple} [dr: #1]}\fi}
\nc{\BP}{\mathbb{P}}
\nc{\BM}{\mathbb{M}}
\nc{\bbapx}{\bb^{\rm apx}}
\nc{\bbapxs}[1]{\bb^{\rm apx, {#1}}}

\nc{\fools}[3]{\MF_{#3}({#1}, {#2})}
\nc{\fool}[2]{\MF({#1},{#2})}
\nc{\clip}[2]{{\rm clip}\left[ \left. {#1} \right| {#2} \right]}
\nc{\imax}{\omega}
\DMO{\conv}{conv}
\nc{\MH}{\mathcal{H}}
\nc{\CH}{\mathscr{H}}
\nc{\CB}{\mathscr{B}}
\nc{\cD}{\mathscr{D}}
\nc{\MC}{\mathcal{C}}
\nc{\st}{\star}
\newcommand{\Pp}{\mathbb{P}}
\nc{\lng}{\langle}
\nc{\rng}{\rangle}
\DMO{\OOPT}{opt}
\nc{\dopt}[2]{\ell_{\OOPT}({#1},{#2})}
\nc{\grad}{\nabla}
\nc{\MG}{\mathcal{G}}
\nc{\MP}{\mathcal{P}}
\nc{\PP}{\mathbb{P}}
\nc{\TT}{\mathbb{T}}
\nc{\TTmax}{\TT_{\max}}
\DMO{\Ham}{Ham}
\DMO{\Gap}{Gap}
\DMO{\GD}{GD}
\DMO{\GDA}{GDA}
\DMO{\EG}{EG}
\DMO{\OGDA}{OGDA}
\DMO{\Unif}{Unif}
\DMO{\Tr}{Tr}
\nc{\ul}{\underline}
\nc{\ol}{\overline}
\nc{\Qu}{\ul{Q}}
\nc{\Qo}{\ol{Q}}
\nc{\Ro}{\ol{R}}
\nc{\Vu}{\ul{V}}
\nc{\Vo}{\ol{V}}
\nc{\RanQ}{\Delta Q}
\nc{\RanV}{\Delta V}
\nc{\clipQ}{\Delta \breve{Q}}
\nc{\frzQ}{\Delta \mathring{Q}}
\nc{\clipV}{\Delta \breve{V}}
\nc{\clipdelta}{\breve{\delta}}
\nc{\cliptheta}{\breve{\theta}}
\nc{\delmin}{\Delta_{{\rm min}}}
\nc{\delmins}[1]{\Delta_{{\rm min},{#1}}}
\nc{\gapfinal}[1]{\max \left\{ \frac{\frzQ_{{#1}}^{k^\st}(x,a)}{2H}, \frac{\delmin}{4H} \right\}}
\nc{\post}[2]{R({#1}; {#2})}
\nc{\posts}[3]{R_{#3}({#1}; {#2})}
\nc{\MAJ}{\mathsf{MAJ}}
\nc{\Dnull}{D^{\circ}}

\nc{\BKW}{\mathtt{BKW}}
\nc{\Dec}{\mathtt{Dec}}
\nc{\delreg}{\delta_{\mathsf{reg}}}
\nc{\delreal}{\delta_{\mathsf{real}}}
\nc{\Sreg}{S_{\mathsf{Reg}}}
\nc{\Treg}{T_{\mathsf{Reg}}}
\nc{\PC}{\mathtt{PC}}
\nc{\alphaPC}{\alpha_{\mathsf{PC}}}
\nc{\TPC}{T_{\mathsf{PC}}}
\nc{\SPC}{S_{\mathsf{PC}}}
\nc{\delsmall}{{\delta_{\mathsf{small}}}}
\nc{\CL}{\mathtt{ContrastLearn}}
\nc{\Select}{\mathtt{Select}}
\nc{\piunif}{{\pi_{\mathsf{unif}}}}
\nc{\picov}{{\pi_{\mathsf{cov}}}}
\nc{\ZZ}{\mathbb{Z}}
\nc{\sk}{{\mathsf{sk}}}
\nc{\Enc}{\mathtt{Enc}}
\nc{\EntLPN}{\mathtt{EntangleLPN}}
\nc{\mureg}{{\mu_{\mathsf{reg}}}}
\nc{\PPE}{\mathtt{PPE}}
\nc{\FQI}{\mathtt{FQI}}
\nc{\False}{\mathtt{False}}
\nc{\True}{\mathtt{True}}
\nc{\epreg}{\epsilon_{\mathsf{reg}}}
\nc{\algnst}[1]{\begin{align*}#1\end{align*}}
\nc{\algn}[1]{\begin{align}#1\end{align}}
\nc{\matx}[1]{\left(\begin{matrix}#1\end{matrix}\right)}

\nc{\pimix}{{\pi_{\mathsf{mix}}}}
\nc{\BPC}{{B_{\mathsf{PC}}}}
\nc{\size}{\mathrm{size}}
\nc{\OLIVE}{\texttt{OLIVE}}
\nc{\RP}{\textsf{RP}}
\nc{\cprp}{c_{\mathsf{PRP}}}
\nc{\Mtoy}{\MM_{\mathsf{toy}}}
\nc{\Brute}{\mathtt{Brute}}

\nc{\nuu}{\nu}

\nc{\bel}[1]{\mathbf{b}({#1})}
\nc{\nbel}[1]{\bar{\mathbf{b}}({#1})}
\nc{\sbel}[2]{\mathbf{b}'_{#1}({#2})}
\nc{\nsbel}[2]{\bar{\mathbf{b}}'_{#1}({#2})}

\nc{\bv}{\mathbf{v}}
\nc{\bone}{\mathbf{1}}
\nc{\bX}{\mathbf{X}}
\nc{\be}{\mathbf{e}}
\nc{\bY}{\mathbf{Y}}
\nc{\bG}{\mathbf{G}}
\nc{\bz}{\mathbf{z}}
\nc{\bw}{\mathbf{w}}
\nc{\bA}{\mathbf{A}}
\nc{\bJ}{\mathbf{J}}
\nc{\bK}{\mathbf{K}}
\nc{\bb}{\mathbf{b}}
\nc{\ba}{\mathbf{a}}
\nc{\bs}{\mathbf{s}}
\nc{\bzero}{\mathbf{0}}
\nc{\bi}{\mathbf{i}}
\nc{\Edistinct}{\ME^{\mathsf{distinct}}}
\nc{\bc}{\mathbf{c}}
\nc{\bC}{\mathbf{C}}
\nc{\BR}{\mathbb R}
\nc{\BA}{\mathbb{A}}
\nc{\SA}{\mathscr{A}}
\nc{\BC}{\mathbb C}
\nc{\bx}{\mathbf{x}}
\nc{\bS}{\mathbf{S}}
\nc{\bM}{\mathbf{M}}
\nc{\bR}{\mathbf{R}}
\nc{\bN}{\mathbf{N}}
\nc{\by}{\mathbf{y}}
\nc{\sy}{y}
\nc{\sx}{x}

\nc{\cF}{\mathcal{F}}
\nc{\cE}{\mathcal{E}}
\nc{\MO}{\mathcal O}
\nc{\MQ}{\mathcal{Q}}
\nc{\CO}{\mathscr{O}}
\nc{\MU}{\mathcal{U}}
\nc{\ME}{\mathcal{E}}
\nc{\MN}{\mathcal{N}}
\nc{\MK}{\mathcal{K}}
\nc{\MM}{\mathcal{M}}
\nc{\MS}{\mathcal{S}}
\nc{\MT}{\mathcal{T}}
\nc{\BF}{\mathbb F}
\nc{\BQ}{\mathbb Q}
\nc{\MX}{\mathcal{X}}
\nc{\MA}{\mathcal{A}}
\nc{\MD}{\mathcal{D}}
\nc{\MB}{\mathcal{B}}
\nc{\MZ}{\mathcal{Z}}
\nc{\MJ}{\mathcal{J}}
\nc{\MW}{\mathcal{W}}
\nc{\MR}{\mathcal{R}}
\nc{\MY}{\mathcal{Y}}
\nc{\ML}{\mathcal{L}}
\nc{\BZ}{\mathbb Z}
\nc{\BN}{\mathbb N}
\nc{\ep}{\epsilon}
\nc{\gapfn}[1]{\varepsilon_{#1}}
\nc{\ggapfn}[2]{\varphi_{#1}({#2})}
\nc{\epsahk}{\gapfn{0}}
\nc{\BH}{\mathbb H}
\nc{\BG}{\mathbb{G}}
\nc{\D}{\Delta}
\nc{\MF}{\mathcal{F}}
\nc{\One}{\mathbbm{1}}
\nc{\bOne}{\mathbf{1}}
\nc{\Aopt}{\mathcal{A}^{\rm opt}}
\nc{\Amul}{\mathcal{A}^{\rm mul}}

\nc{\SP}{\mathsf P}
\nc{\SQ}{\mathsf Q}

\nc{\DO}{\accentset{\circ}{\D}}
\nc{\mf}{\mathfrak}
\nc{\mfp}{\mathfrak{p}}
\nc{\mfq}{\mf{q}}
\nc{\Sp}{\mbox{Spec}}
\nc{\Spm}{\mbox{Specm}}
\nc{\hookuparrow}{\mathrel{\rotatebox[origin=c]{90}{$\hookrightarrow$}}}
\nc{\hookdownarrow}{\mathrel{\rotatebox[origin=c]{-90}{$\hookrightarrow$}}}
\nc{\hra}{\hookrightarrow}
\nc{\tra}{\twoheadrightarrow}
\nc{\sgn}{{\rm sgn}}
\nc{\aut}{{\rm Aut}}
\nc{\Hom}{{\rm Hom}}
\nc{\img}{{\rm Im}}
\DMO{\id}{Id}
\DMO{\supp}{supp}
\DMO{\KL}{KL}
\nc{\kld}[2]{\KL({#1}||{#2})}
\nc{\ren}[2]{D_2({#1}||{#2})}
\nc{\chisq}[2]{\chi^2({#1}||{#2})}
\nc{\tvd}[2]{D_{\mathsf{TV}}\left({#1}, {#2}\right)}
\nc{\hell}[2]{H^2({#1}, {#2})}
\DMO{\BSS}{BSS}
\DMO{\BES}{BES}
\DMO{\BGS}{BGS}
\DMO{\poly}{poly}
\nc{\indep}{\perp}
\DMO{\sink}{sink}
\DMO{\nosink}{nosink}
\nc{\sinks}{s^{\sink}}
\nc{\sinkobs}{o^{\sink}}
\nc{\fp}[1]{\MP_1({#1})}
\nc{\BO}{\mathbb{O}}
\nc{\BT}{\mathbb{T}}

\nc{\RR}{\mathbb{R}}
\nc{\NN}{\mathbb{N}}
\nc{\Gradient}{\nabla}
\DMO{\diag}{diag}
\nc{\norm}[1]{\left \lVert #1 \right \rVert}
\DMO*{\EE}{\mathbb{E}}
\nc{\LPN}{\mathsf{LPN}}
\DMO{\Ber}{Ber}
\nc{\Regress}{\mathtt{Regress}}
\nc{\LFC}{\mathtt{LearnFromCorr}}
\nc{\RegressAlg}{\mathtt{RegressAlg}}
\nc{\DrawTraj}{\mathtt{DrawTrajectory}}
\nc{\pizero}{{\pi_{\mathsf{zero}}}}
\nc{\Tred}{{T_{\mathsf{red}}}}
\nc{\epred}{{\epsilon_{\mathsf{red}}}}
\nc{\TriAlg}{\mathtt{GenerateTriangleLPN}}
\nc{\Alg}{\mathtt{Alg}}
\nc{\AffSample}{\mathtt{AffSample}}
\nc{\br}{\mathbf{r}}
\nc{\TV}{{\mathsf{TV}}}
\DMO{\Law}{Law}
\DMO{\Sym}{Sym}
\nc{\bu}{\mathbf{u}}
\nc{\Reg}{\mathtt{Reg}}
\nc{\Breg}{B_{\mathsf{Reg}}}
\DMO{\dc}{dc}

\DMO{\PR}{Pr}
\renewcommand{\Pr}{\PR}
\DMO*{\Prr}{Pr}
\nc{\E}{\mathbb{E}}
\nc{\ra}{\rightarrow}
\renewcommand{\t}{\top}

\nc{\fq}{\mathfrak{q}}
\nc{\laplace}{\Delta}
\nc{\pavg}{\overline p}
\nc{\CPI}{C_{\mathsf{PI}}}
\nc{\convolve}{\star}
\nc{\normal}{\mathcal{N}}
\nc{\phat}{\widehat p}
\nc{\tstar}{t^\star}
\nc{\dd}{\mathrm{d}}
\nc{\nutil}{\widetilde \nu}

\DeclarePairedDelimiter{\prn}{(}{)}

\newcommand{\Dkl}[2]{D_{\mathsf{KL}}\prn*{#1\|#2}}
\newcommand{\Dtv}[2]{D_{\mathsf{TV}}\prn*{#1,#2}}
\nc{\pstar}{p^\star}
\nc{\ind}[1]{^{\footnotesize (#1)}}
\nc{\cL}{\mathcal{L}}
\nc{\qtil}{\widetilde{q}}
\nc{\indic}{\mathbbm{1}}
\nc{\pbar}{\ol p}
\nc{\qbar}{\ol q}
\nc{\qstar}{q^\star}
\nc{\epsscore}{\epsilon_{\mathsf{score}}}

\nc{\khat}{\hat \kappa}
\nc{\Zhat}{\hat Z}

\newcommand{\linearalg}{\mathsf{LinTiltSampler}}
\newcommand{\basealg}{\mathsf{UnadjustedSampler}}
\newcommand{\sstar}{s^\star}
\newcommand{\psdalg}{\mathsf{PSDTiltSampler}}
\newcommand{\partitionalg}{\mathsf{EstimateNormalization}}
\nc{\epfinal}{\epsilon_{\mathsf{final}}}
\nc{\cS}{\mathcal{S}}
\nc{\xtil}{\widetilde x}
\nc{\ztil}{\widetilde z}
\nc{\twist}[1]{p_{#1}}
\nc{\htwist}[1]{\phat_{#1}}
\nc{\Cnorm}{C_{\mathsf{norm}}}
\nc{\Clip}{C_{\mathsf{lip}}}
\nc{\Ckl}{C_{\mathsf{KL}}}
\nc{\WW}{\mathsf{W}_2}
\DMO{\rank}{rank}
\nc{\rr}{\mathbf{r}}
\nc{\NP}{\mathrm{NP}}
\nc{\BPP}{\mathrm{BPP}}
\arxiv{
\usepackage{natbib}
\date{}
}

\begin{document}

\maketitle

\vspace{-3em}
\begin{center}
\large
\setlength{\tabcolsep}{20pt}
\begin{tabular}{ccc}
\makecell{Ankur Moitra \\ \small{\texttt{moitra@mit.edu}}}
&
\makecell{Andrej Risteski \\ \small{\texttt{aristesk@andrew.cmu.edu}}}
&
\makecell{Dhruv Rohatgi \\ \small{\texttt{drohatgi@mit.edu}}}
\end{tabular}
\end{center}
\vspace{1em}

\begin{abstract}%
  Inference-time algorithms are an emerging paradigm in which pre-trained models are used as subroutines to solve downstream tasks. Such algorithms have been proposed for tasks ranging from inverse problems and guided image generation to reasoning. However, the methods currently deployed in practice are heuristics with a variety of failure modes---and we have very little understanding of when these heuristics can be efficiently improved.   

  In this paper, we consider the task of sampling from a reward-tilted diffusion model---that is, sampling from $p^{\star}(x) \propto p(x) \exp(\rr(x))$---given a reward function $\rr$ and pre-trained diffusion oracle for $p$. We provide a fine-grained analysis of the computational tractability of this task for quadratic rewards $\rr(x) = x^\t A x + b^\t x$. We show that linear-reward tilts are always efficiently sampleable---a simple result that seems to have gone unnoticed in the literature. We use this as a building block, along with a conceptually new ingredient---the Hubbard-Stratonovich transform---to provide an efficient algorithm for sampling from low-rank positive-definite quadratic tilts, i.e. $\rr(x) = x^\t A x$ where $A$ is positive-definite and of rank $O(1)$. For negative-definite tilts, i.e. $\rr(x) = - x^\t A x$ where $A$ is positive-definite, we prove that the problem is intractable even if $A$ is of rank 1 (albeit with exponentially-large entries).  
\end{abstract}

\colt{
}
\section{Introduction}

It is increasingly common to use pre-trained generative models as components within algorithms for more complex downstream tasks. Broadly, such algorithms are termed as \emph{inference-time} or \emph{meta-generation} algorithms \citep{welleck2024decoding}. Examples of tasks that can be framed in this paradigm include \emph{inverse problems} in the sciences \citep{bruna2024provable}, guidance to perform conditional sampling in image generation \citep{dhariwal2021diffusion}, and tilting a distribution by a (trained or pre-specified) reward in reasoning tasks \citep{korbak2205rl, geuter2025guided} and even protein design \citep{lisanza2025multistate, hartman2025controllable}.  

In this paper, we focus on the algorithmic task of steering a pre-trained diffusion model according to a given reward function. Specifically, given a diffusion model for a \emph{base distribution} $p(x)$ (which provides access to the scores of convolutions of $p$ with Gaussian noise), and a reward function $\rr(x)$, our task is to sample from the \emph{tilted distribution}
\[\pstar(x) \propto p(x) \exp(\rr(x)).\]
%the problem of sampling from reward-tilted diffusion models, namely sampling from by $p^*(x) \propto p(x)\exp(r(x))$, assuming we have access to the scores of convolutions of $p$ with Gaussian noise, as well the reward $r$. If we pre-trained a diffusion model to sample from $p$, this is exactly the kind of information about $p$ we would have access to.
 For any reward function $\rr$, $\pstar$ has a natural variational interpretation as the solution to the KL-regularized optimization problem $\argmax_q \EE_q[\rr] - \Dkl{q}{p}$ \citep{korbak2205rl}. Moreover, with appropriate choices of the reward function---ranging from simple quadratics to complex pre-trained reward models---the task of sampling from $\pstar$ formalizes many concrete practical problems, including inpainting \citep{karan2025reguidance}, posterior inference under noisy measurements \citep{bruna2024provable}, and human preference alignment \citep{singhal2025a}. %This setting is general enough to encapsulate many settings of interest. For example, if $r$ is the log-likelihood of a trained classifier, by Bayes rule, the above corresponds to sampling from the class-conditional distribution ; if $r$ is the log-likelihood of a noisy measurement, by Bayes rule, the above distribution corresponds to sampling from the posterior over clean measurements in inverse problems. 
Empirically, the algorithms used are often heuristics, and have known failure modes \citep{chidambaram2024does}. Theoretically, the algorithmic landscape for this task, and the fundamental computational barriers, remain largely unexplored---even when the reward function is very simple.

In this paper, we focus on the family of quadratic reward functions $\rr(x) = x^\t A x + b^\t x$. With this family of rewards, the steering task already encapsulates several of the preceding applications \citep{karan2025reguidance,bruna2024provable}. For example, posterior inference with linear measurements and Gaussian noise corresponds to steering with the log-density of an appropriate Gaussian: $\rr(x) = -\frac{1}{2\sigma^2}\norm{Mx-y}_2^2$. Moreover, quadratic bonuses of the form $\rr(x) = x^\t \Sigma^{-1} x$, for positive-definite $\Sigma$, are commonly used in applications such as rare-event sampling \citep{asmussen2011efficient} and reinforcement learning \citep{tuyls2025representation}, to steer towards (or optimize for) rare or novel generations.

%From an applications point of view, in linear inverse problems with Gaussian noise,  the reward is a negative-definite quadratic $r(x) = -\frac{1}{2\sigma^2}\|Ax-b\|^2$; in rare-event sampling and reinforcement learning rewards of the form $r(x) = \|Px\|^2$ can be used to skew the distribution towards rarer events and as an exploration bonus respectively \citep{}. 
From a theoretical perspective, steering with quadratic rewards is computationally intractable with no further assumptions \citep{gupta2024diffusion, bruna2024provable}. Indeed, if the base distribution $p(x)$ is uniform over the \emph{discrete hypercube}, then $\pstar(x) \propto \exp(x^\t A x + b^\t x)$ is precisely a classical \emph{Ising model} \citep{ising1925beitrag}, and a seminal line of work has shown that exact and approximate sampling from an Ising model can be intractable \citep{jerrum1993polynomial,sly2012computational,galanis2016inapproximability}. However, there is also a rich literature on efficient algorithms for sampling from Ising models with special structure \citep{jerrum1993polynomial,eldan2022spectral,koehler2022sampling,chen2022localization}. Can these sorts of structural assumptions help explain when steering a general base measure $p$ accessible via a diffusion model is indeed tractable?%Can similar guarantees be achieved even in the setting where $p$ is a general base measure (accessible via a score oracle)---in order to develop provable inference-time steering algorithms?

%Finally, by way of connections to the classical topic of sampling from Gibbs measures, sampling from distributions of the form $p^{\star}(x) \propto \exp(x^T A x + b^T x)$ over the discrete hypercube is the classical problem of sampling from Ising models, with a rich history in statistical physics and computer science. Thus, our question can be thought of as an inference-time analogue of these questions---where there is additionally a base measure accessible via a score oracle.   
%---that is, $p^*(x) \propto p(x) \exp(x^T A x + b^T x)$ for some suitable $A, b$. 
%Recent work by \cite{bruna2024provable} studied the problem in the context of inverse problems with linear measurements---as well as proposed algorithms which work in regimes in which the signal-to-noise ratio is favorable. 
\subsection{Contributions}

In this paper, we provide a fine-grained understanding of the computational landscape of steering diffusion models with quadratic rewards $\rr(x) := x^\t A x + b^\t x$. Through the lens of the \emph{rank} of the quadratic form (i.e. $\rank(A)$)---a fundamental quantity in the special case of sampling from Ising models \citep{koehler2022sampling}---we delineate regimes in which the task is computationally tractable, and regimes in which it is not. Precisely, we show the following:

\paragraph{Linear rewards admit an efficient sampler (\cref{sec:linear}).} If the reward is a linear function $\rr(x) = b^\t x$ for some vector $b \in \RR^d$ (in other words, $\rank(A)=0$), we can efficiently sample from the tilt $p^{\star}$. This relatively simple result was seemingly missed in prior literature, and is a consequence of the fact that the scores of $\pstar$ have a simple closed-form expression in terms of the reward function and scores of the base distribution $p$. 

\paragraph{Negative-definite quadratic rewards induce intractability, even for rank-1 matrices (\cref{sec:nsd}).} If $\rr(x) = x^\t A x$, where $A$ is a rank-$1$ negative semi-definite matrix, then the task of sampling from $p^{\star}$ is computationally intractable, assuming $\mathrm{NP}\nsubseteq \mathrm{BPP}$. %This result strengthens results from \citep{bruna2024provable} who show intractability for general negative-definite tilts. 

\paragraph{Positive-definite, low-rank quadratic rewards admit an efficient sampler (\cref{sec:psd}).} If $\rr(x) = x^\t A x$, where $A$ is a rank-$O(1)$ positive semi-definite matrix, then there is an algorithm that samples from $p^*$ (approximately, in Wasserstein distance) in polynomial time.\footnote{We remark that the runtime also scales polynomially in $\norm{A}_2$. This leaves a conceptual gap, since our hardness result for negative-definite rewards does not rule out an algorithm with similar scaling. We believe such an algorithm is unlikely to exist, but we defer resolution of this question to future work.} For completeness, we also show that without the low-rank assumption, the problem becomes computationally intractable (\cref{sec:psd-high-rank}). The algorithmic result is based on the idea of using the \emph{Hubbard-Stratonovich} transform \citep{hubbard1959calculation} to construct a lifting of the target distribution (i.e. introduce a new variable). We show how to sample from this lifted distribution using \emph{sampling from linear-reward tilts} as a subroutine. We believe this result is of additional conceptual interest as it shows that sampling from linear-reward tilts can be a useful building block for designing inference-time algorithms even for more complex reward models.

\subsection{Related work} 

Several prior works have studied the problem of provably steering diffusion models with a quadratic reward function $r(x) = -\frac{1}{2\sigma^2} \|Ax - b\|^2$ \citep{gupta2024diffusion,bruna2024provable,xun2025posterior,parulekar2025efficient}, motivated by the task of linear inverse problems, i.e. posterior inference with noisy linear measurements. \cite{gupta2024diffusion} show that steering with general (i.e. potentially high-rank) negative-definite quadratic rewards is cryptographically hard, and \cite{bruna2024provable} show that the same problem is intractable via reduction from the problem of sampling Ising models. Our result in \cref{sec:nsd} strengthens these by showing that intractability holds even when the quadratic form is rank-$1$.

\cite{bruna2024provable} show that the hardness can be circumvented when $\sigma$ is sufficiently small. In the context of inverse problems, $A$ corresponds to a measurement operator, and $\sigma$ to a signal-to-noise ratio; the small-$\sigma$ regime is easier since it makes $\pstar$ more log-concave. The main tool they use is the Polchinsky flow \citep{bauerschmidt2024stochastic}. \cite{xun2025posterior} develop an efficient algorithm for the same problem, without assumptions on $A$ and $\sigma$, but require that the base distribution $p$ satisfies a condition called ``local log-concavity''. \cite{parulekar2025efficient} develop an efficient algorithm with no assumptions, but it (necessarily) has no guarantee of closeness in total variation or Wasserstein distance. % \dhruv{also \cite{gupta2024diffusion,xun2025posterior}}

Zooming out, \cite{chidambaram2024does} analyze a popular heuristic for guidance---i.e. sampling from a class-conditioned diffusion model---and show some natural examples in which it has the intended behavior, and some examples of failure modes. \cite{karan2025reguidance} provide a wrapper for existing heuristics, which admits a theoretical guarantee that is weaker than approximate sampling. \cite{rohatgi2025taming} consider sampling from tilted distributions when $p$ is the law of an autoregressive language model, and reward access is augmented by a ``process reward'' that estimates the quality of partial generations. %They focus on the propagation of errors in the given process reward.  

Finally, at a technical level, our application of Hubbard-Stratonovich in the PSD setting is inspired by that of \cite{koehler2022sampling} for sampling from (approximately) low-rank Ising models, given access to their unnormalized density. Their algorithm uses the transform in conjunction with techniques from Markov Chains and variational inference, whereas ours requires building on our algorithm for linear rewards.

\section{Preliminaries and notation} 

\paragraph{Notation.} For a set $S$, let $\Delta(S)$ denote the space of distributions over $S$. For $R>0$, let $\CB_{d,2}(R) := \{x\in\RR^d:\norm{x}_2 \leq R\}$. For a distribution $p$, $\supp(p)$ denotes its support. For distributions $p,q$, $\TV(p,q)$ denotes total variation distance and $\WW(p,q)$ denotes Wasserstein-$2$ distance. We let $\normal(\mu,\Sigma)$ denote the Gaussian distribution with mean $\mu$ and covariance $\Sigma$.

\begin{definition}[Noised distribution]
For any distribution $q \in \Delta(\RR^d)$, for any $\sigma \in [0,1]$, we define $q_\sigma \in \Delta(\RR^d)$ as the law of $\sqrt{1-\sigma^2}X+\sigma Z$ where $X \sim q$ and $Z \sim \normal(0,I_d)$. %Equivalently, $q_\sigma$ is the convolution of $q$ with $\normal(0,\sigma^2 I_d)$.\dhruv{switch to OU}
\end{definition}

\paragraph{Formal setting.} Let $p \in \Delta(\RR^d)$ be a \emph{base distribution} and let $\rr(x): \RR^d \to \RR$ be a reward function. Our goal is to (approximately) sample from the \emph{tilted distribution} $\pstar \in \Delta(\RR^d)$ defined by: 
\[\pstar(x) \propto p(x)\exp(\rr(x)),\] 
which is well-defined whenever $\EE_{x\sim p}[\exp(\rr(x))]<\infty$. We make the following standard boundedness assumption \citep{de2022convergence,chen2023sampling}, which ensures that $\pstar$ is well-defined for any continuous reward; our algorithms will have runtime polynomial in the bound $\Cnorm$:

\begin{assumption}[Boundedness]\label{ass:score}
Let $\Cnorm \geq 1$. We assume that $\sup_{x\in\supp(p)} \norm{x}_2 \leq \Cnorm$.%, and $\log(\Dkl{p}{\normal(0,I_d)}) \leq \Ckl$.
\end{assumption} 

We access $p$ via the following oracle, which is exactly the object one would have access to if a diffusion model was (pre-)trained on the distribution $p$. 

\begin{assumption}[Score oracle]
For any $\sigma \in (0,1)$ and $x \in \RR^d$, we can query $s_\sigma(x) := \grad \log p_\sigma(x)$. 
\end{assumption}

Note that we assume exact access to the scores, following prior work \citep{bruna2024provable,parulekar2025efficient}; understanding the effect of errors is an interesting open question.%; we discuss open questions relating to the effect of errors in \dhruv{conclusion}.
 
%The only assumption we make is that the base distribution $p$ is supported in a Euclidean ball of bounded radius, and our algorithmic results will scale polynomially in this bound:%For our algorithmic results, we make the following standard assumptions on the base distribution $p$. Note that we assume that the norm of samples from $p$ is bounded almost surely (rather than in expectation, as in some prior work \citep{chen2023sampling}); this ensures that the tilted distribution $\pstar$ is well-defined, and can likely be relaxed to an assumption that the tails of $p$ enjoy exponential decay.

As shown by \cite{chen2023sampling}, under \cref{ass:score}, the score oracle enables efficient approximate sampling from the \emph{base distribution} $p$ with small Wasserstein-$2$ error:

\begin{theorem}[\cite{chen2023sampling}]\label{thm:ccl}
Let $d,\Cnorm \in \NN$ and $\epsilon>0$. Fix $q \in \Delta(\RR^d)$. Suppose that $\supp(q) \subseteq \CB_{d,2}(\Cnorm)$. There is a $\poly(d, \epsilon^{-1}, \Cnorm)$-time algorithm $\basealg$ that, given parameters $\epsilon,\Cnorm$ as well as query access to $\grad \log q_\sigma(x)$ for any $x\in\RR^d$ and $\sigma\in(0,1)$, produces a sample from distribution $\qtil$ with $\WW(q,\qtil) \leq \epsilon$ and $\supp(\qtil) \subseteq \CB_{d,2}(\Cnorm)$.\footnote{The second property is not explicitly stated by \cite{chen2023sampling}, but it is immediate since projection onto $\CB_{d,2}(\Cnorm)$ is contractive in $\ell_2$.}
\end{theorem}

Under stronger assumptions (e.g. Lipschitzness of the scores), the approximation in Wasserstein distance can be upgraded to approximation in total variation \citep{chen2023sampling}. However, in this work we focus on the minimal assumptions described above, and seek to approximately sample from $\pstar$ in Wasserstein.

\section{Steering with linear rewards is tractable}\label{sec:linear}
\allowdisplaybreaks
In this section we prove \cref{lemma:linearalg-analysis}, which states that if $\rr$ is a linear function, there is an efficient approximate sampler $\linearalg$ (\cref{alg:linear-tilt}) for $\pstar$. The following definition will be convenient:%samples from \emph{linear} tilts can be efficiently generated.

\begin{definition}\label{def:linear-tilt}
Fix $v \in \RR^d$. We define $p(\cdot;v) \in \Delta(\RR^d)$ by $p(x;v) \propto p(x) e^{\langle x, v\rangle}.$%\dhruv{switch to this notation throughout}
\end{definition}

\begin{theorem}\label{lemma:linearalg-analysis}
Suppose that \cref{ass:score} holds. Let $v \in \RR^d$ and $\epsilon > 0$. The output $\xtil \gets \linearalg((s_\sigma)_\sigma, v, \epsilon,\Cnorm)$ has law $\phat(\cdot;v)$ satisfying $\WW(\phat(\cdot;v),p(\cdot;v)) \leq \epsilon$ and $\supp(\phat(\cdot;v))\subseteq \CB_{d,2}(\Cnorm)$. Moreover, the time complexity of the algorithm is at most $\poly(d,\epsilon^{-1},\Cnorm)$.
\end{theorem}

\begin{algorithm}[t]
\caption{$\linearalg$: Steering diffusion model with linear reward}
\label{alg:linear-tilt}
\begin{algorithmic}[1]
  \State \textbf{input:} Score functions $(s_\sigma)_{\sigma \in (0,1)}$, tilt vector $v \in \RR^d$, error tolerance $\epsilon>0$, norm bound $\Cnorm \geq 1$.
  \State For each $\sigma \in (0,1)$, define $\sstar_\sigma: \RR^d \to \RR^d$ by
  \[\sstar_\sigma(x) := \frac{v}{\sqrt{1-\sigma^2}} + s_\sigma\left(x + \frac{\sigma^2}{\sqrt{1-\sigma^2}} v\right).\]
  \State Compute $\xtil \gets \basealg((\sstar_\sigma)_\sigma,\epsilon,\Cnorm)$.\Comment{\cref{thm:ccl}}
  \State \textbf{return:} $\xtil$.
\end{algorithmic}
\end{algorithm}

The key lemma facilitating this result is the following: 

\begin{lemma}\label{lemma:pstar-sigma}
For any $x \in \RR^d$ and $\sigma \geq 0$, it holds that
\[\grad \log p_\sigma(x;v) = \frac{v}{\sqrt{1-\sigma^2}} + \grad \log p_\sigma\left(x+\frac{\sigma^2}{\sqrt{1-\sigma^2}} v\right)\]
\end{lemma}

\begin{proof}
We explicitly compute the density $p_\sigma(x;v)$. Set $t := \sqrt{1-\sigma^2}$. Then:
\begin{align*}
p_\sigma(x;v)
&= \frac{t^{-d}}{(2\pi\sigma^2)^{d/2}} \int_{\RR^d} p(t^{-1}y;v) e^{-\norm{y - x}_2^2 / (2\sigma^2)} \, \dd y \\ 
&\propto \int_{\RR^d} p(t^{-1}y) e^{\langle t^{-1}y, v\rangle -\norm{y - x}_2^2 / (2\sigma^2)} \, \dd y \\ 
&= e^{\langle t^{-1} x,v\rangle} \int_{\RR^d} p(t^{-1}y) e^{\langle t^{-1}(y-x), v\rangle -\norm{y - x}_2^2 / (2\sigma^2)} \, \dd y \\ 
&= e^{\langle t^{-1} x,v\rangle}\int_{\RR^d} p(t^{-1}y) e^{-\norm{y - x - t^{-1}\sigma^2 v}_2^2 / (2\sigma^2)} \, \dd y \\
&= e^{\langle t^{-1} x,v\rangle} p_\sigma(x+t^{-1}\sigma^2 v)
\end{align*}
where the fourth equality is by completing the square. It follows that
\[\grad \log p_\sigma(x;v) = \frac{v}{\sqrt{1-\sigma^2}} + \grad \log p_\sigma\left(x+\frac{\sigma^2}{\sqrt{1-\sigma^2}} v\right)\]
as claimed.
\end{proof}

The implication of this lemma is that using the score oracle for $p$, we can efficiently simulate a score oracle for $\pstar := p(\cdot;v)$. We can plug this result into \cref{thm:ccl} (due to \cite{chen2023sampling}), which states that for any distribution $q$ with bounded support, given query access to $\grad \log q_\sigma$ for any $\sigma \in (0,1)$, there is an efficient algorithm that approximately samples from $q$. The proof of \cref{lemma:linearalg-analysis} is essentially immediate:

\begin{namedproof}{\cref{lemma:linearalg-analysis}}
By \cref{lemma:pstar-sigma}, the functions $\sstar_\sigma$ defined in \cref{alg:linear-tilt} satisfy $\sstar_\sigma(x) = \grad \log p_\sigma(x;v)$ for all $x\in\RR^d$ and $\sigma\in(0,1)$. Since $\supp(p) \subseteq \CB_{d,2}(\Cnorm)$, it is immediate that $\supp(p(\cdot;v)) \subseteq \CB_{d,2}(\Cnorm)$. The claim then follows from \cref{thm:ccl}.
\end{namedproof}

\iffalse

\begin{corollary}
Let $L,M,K \in \NN$ and $\epsilon>0$. Suppose that $x \mapsto \grad \log p_\sigma(x)$ is $L$-Lipschitz for all $\sigma \geq 0$, that $\EE_{\pstar}[\norm{x}_2^2] \leq M$, and that $\Dkl{\pstar}{\normal(0,I_d)} \leq K$. There is a $\poly(d, L, M, 1/\epsilon, \log(K/\epsilon))$-time algorithm that, in the above access model, produces a sample from distribution $\qtil$ with $\Dtv{\pstar}{\qtil} \leq \epsilon$.
\end{corollary}

\begin{proof}
We apply \cref{thm:ccl} with $q := \pstar$, using \cref{lemma:pstar-sigma} and the assumed access to $v$ and $\grad \log p_\sigma$ to simulate queries to $\grad \log \pstar_\sigma$. It is straightforward to check that the conditions of \cref{thm:ccl} are met.
\end{proof}
\fi

\section{Steering with low-rank negative-definite rewards is hard}\label{sec:nsd}

In this section, we show that the problem of steering a diffusion model with a low-rank negative-definite reward---specifically, $\rr(x) = x^\t A x$, where $A$ is a negative semi-definite matrix of rank $1$---is computationally intractable, assuming a standard conjecture ($\NP \not\subseteq \BPP$) from computational complexity. This refines a result of  \cite{bruna2024provable} (who showed this claim for general negative semi-definite matrices $A$) and combines elements of their analysis with a result of \cite{koehler2022sampling} (who showed that sampling from rank-$1$ Ising models is computationally intractable).
%Negative semi-definite rewards naturally appear in many tasks---for example, posterior sampling with linear measurements. 
%We show that the problem is intractable even with the strong additional restriction that $A$ is rank-$1$:%, the problem is still intractable unless $\NP \subseteq \BPP$:

%The result we show is as follows: 

\begin{theorem}\label{thm:nsd-hardness}
Suppose that there is a randomized algorithm $\MA$ with the following property. For any integer $d \in \NN$, any distribution $p \in \Delta(\RR^d)$ satisfying \cref{ass:score} with parameter $\Cnorm \geq 1$, and any rank-$1$, negative semi-definite $A \in \RR^{d\times d}$, the output $\xtil \in \RR^d$ of $\MA((\grad \log p_\sigma)_{\sigma\in(0,1)}, A, \Cnorm)$ has law $\nu$ satisfying $\WW(\nu, \pstar) \leq 1/4$, where $\pstar \in \Delta(\RR^d)$ is the tilted distribution \[\pstar(x) \propto p(x) \exp(x^\t A x).\] Moreover, the time complexity of $\MA((\grad \log p_\sigma)_{\sigma\in(0,1)}, w, \Cnorm)$ is $\poly(d,\Cnorm)$. 

Then, $\NP \subseteq \BPP$. 
\end{theorem}

As a caveat, \cref{thm:nsd-hardness} does not rule out an algorithm with time complexity that also scales polynomially in $\norm{A}_2$. We will prove \cref{thm:nsd-hardness} by reducing from the $\NP$-hard PARTITION problem \citep{karp1975computational}, defined as follows: 

\begin{definition}[PARTITION]
Given integers $w := (a_1,\dots,a_d) \in \mathbb{Z}^d$, the PARTITION problem is to decide whether there exists
$x \in \{\pm 1\}^d$ such that $w^\top x = 0.$
\end{definition}

Given an instance $w \in \ZZ^d$ of the partition problem, we will define a base distribution $p \in \Delta(\RR^d)$ and matrix $A_w \in \RR^{d\times d}$ as follows:

\[
p := \Unif(\{-1,1\}^d) = 2^{-d}\sum_{x\in\{\pm1\}^d}\delta_x \in \Delta(\RR^d),  
\qquad
A_w := -(d+5)\,ww^\t \preceq 0.
\]
For notational convenience, we then define $q_w \in \Delta(\RR^d)$ to be the tilted distribution
\[q_w(x) \propto p(x) \exp(x^\t A_w x).\]
Note that since $p_\sigma$ is a product distribution for each $\sigma\in(0,1)$, and each marginal is a mixture of two Gaussians, the score oracle for $p$ can be simulated efficiently. Moreover, intuitively, the tilted distribution $q_w$ will be concentrated on $x \in \{-1,1\}^d$ with $\langle x,w\rangle \approx 0$. Thus, if $\nu$ is close to $q_w$ in Wasserstein distance, then $\nu$ will be concentrated near such $x$.

The following lemma helps formalize this intuition by lower bounding the mass of the tilted distribution on solutions to the PARTITION problem, in the event that the tilt corresponds to a YES instance of the PARTITION problem. We show:  

\begin{lemma}
\label{lem:mass-on-solutions}
Given a PARTITION instance $w\in\mathbb{Z}^d$, define the set $S_w := \{x\in\{\pm1\}^d : w^\top x=0\}$. Assume $S_w\neq \emptyset$ (i.e.\ the PARTITION instance is a YES instance). Then $q_w(S_w) \;\ge\; \frac{200}{201}.$
\end{lemma}

See \cref{sec:mass-on-solutions} for the proof. With this lemma in hand, we can prove \cref{thm:nsd-hardness}: 

\begin{namedproof}{\cref{thm:nsd-hardness}}
Given a PARTITION instance $w\in\mathbb{Z}^n$, construct $p$ and $A_w$ as above,
and consider the corresponding tilted distribution $\pstar = q_w$ supported on $\{\pm1\}^d$. Set $\Cnorm := \sqrt{d}$ and note that \cref{ass:score} is satisfied. Define the set $S_w := \{x\in\{\pm1\}^n : w^\top x=0\}$ and define
\[
R_{S_w} := \{y\in\RR^n : \sgn(y)\in S_w\}.
\]
For $\sigma\in(0,1)$, let $s_\sigma := \grad \log p_\sigma$, and note that we can efficiently simulate queries to $s_\sigma$ for any $x \in \RR^d$. To solve the PARTITION problem, we will run the sampler $\mathcal{A}((s_\sigma)_{\sigma\in(0,1)},A_w,\Cnorm)$ to obtain $Y\in\RR^n$. We will then round the output to compute $\widehat x := \sgn(Y)\in\{\pm1\}^n$. We will output YES iff $w^\top \widehat x = 0$. %Note that since the prior is just a mixture over the vertices of the hypercube, we can efficiently simulate the responses for the score oracle---so the entire procedure runs in polynomial time.  

We analyze what happens in the YES and NO cases.

If the PARTITION instance is NO, then $S_w=\emptyset$.
In that case, for every $\widehat x\in\{\pm1\}^n$, we have $w^\top \widehat x \neq 0$,
so the above algorithm always outputs NO.

In the YES case, we have \ $S_w\neq\emptyset$. Let $\nu_w$ be the output distribution of $Y$.
By assumption, $\WW(\nu_w,q_w)\leq 1/4$.
Applying \cref{lem:rounding-mass} with $\mu:=q_w$, $\nu:=\nu_w$, and $S:=S_w$, we obtain
\[
\nu_w(R_{S_w}) \;\ge\; q_w(S_w) - (1/4)^2.
\]
By \cref{lem:mass-on-solutions}, $q_w(S_w)\ge \frac{200}{201}$. Plugging in these numbers, we get $
\nu_w(R_{S_w}) > 0.9.$ Thus, this algorithm decides PARTITION with
one-sided error: it always outputs NO on NO instances, and outputs YES on YES instances
with probability at least $0.9$. Moreover, by assumption on $\MA$, the time complexity of this algorithm is $\poly(d)$. It follows that $\NP \subseteq \BPP$.
\end{namedproof}

\section{Steering with low-rank positive-definite rewards is tractable}\label{sec:psd}

In this section, we show that the problem of steering a diffusion model with low-rank \emph{positive-definite} rewards, i.e. with $\rr(x) = x^\t A x$ for positive semi-definite low-rank $A$, is computationally tractable. Without loss of generality, we may write $A = \frac{1}{2} L^\t L$ where $L$ is an $r \times d$ matrix. Then the tilted distribution $\pstar$ is as defined below:%We will show that for any $r = O(1)$, there is a polynomial-time approximate sampler for $\pstar$.

\begin{definition}
Fix $p \in \Delta(\RR^d)$ satisfying \cref{ass:score} with parameter $\Cnorm$. Fix a matrix $L \in \RR^{r \times d}$. We define $\pstar \in \Delta(\RR^d)$ by
\[\pstar(x) := \frac{p(x) e^{\frac{1}{2}\norm{Lx}_2^2}}{Z}\]
where $Z := \EE_{x \sim p}[e^{\frac{1}{2}\norm{Lx}_2^2}]$ is the normalization constant.
%\ar{Maybe say subject to Z being finite}
\end{definition}

Note that $Z$ is finite by \cref{ass:score}, and hence $\pstar$ is well-defined. The main result of this section is the following theorem, which shows that $\psdalg$ (\cref{alg:psd-tilt}) samples from $\pstar$ in polynomial time whenever $r = O(1)$ (and $\Cnorm$ and $\norm{L}_2$ are polynomially bounded):

\begin{theorem}\label{theorem:psd}
Suppose that \cref{ass:score} holds. Let $D\geq 1$ and $\epfinal \in (0,1/2)$. Suppose that $D \geq \sup_{x \in \supp(p)} \norm{Lx}_2$. Then the output $\xtil$ of $\psdalg((s_\sigma)_\sigma, L, D, \Cnorm, \epfinal)$ has law $\mu$ satisfying $\WW(\mu, \pstar) \leq \epfinal$. Moreover, the time complexity of the algorithm is at most \[\poly(d, \Cnorm, \norm{L}_2, D^r, r^r, \epfinal^{-r}).\]
\end{theorem}

The key insight is the following decomposition, leveraging the Hubbard-Stratonovich \citep{hubbard1959calculation} transform:

\begin{definition}
For each $z \in \RR^r$, define $Z(z) := \int_{\RR^d} p(x) e^{\langle Lx,z\rangle}\,\dd x$.
\end{definition}
As before, $Z(z)$ is finite by \cref{ass:score}. %\ar{Again, say this is finite?}

\begin{lemma}\label{lemma:hs-transform}
It holds for all $x \in \RR^d$ that
\[\pstar(x) = \frac{(2\pi)^{-d/2}}{Z} \int_{\RR^r} Z(z) e^{-\frac{1}{2}\norm{z}_2^2} p(x;L^\t z) \, \dd z.\]
\end{lemma}

\begin{proof}
The Hubbard-Stratonovich transform gives
\[e^{\frac{1}{2}\norm{Lx}_2^2} = (2\pi)^{-d/2} \int_z e^{-\frac{1}{2}\norm{z}_2^2 + \langle Lx,z\rangle} \,\dd z.\]
It follows that
\begin{align*}
\pstar(x) 
&= \frac{(2\pi)^{-d/2} \int_{\RR^r} e^{-\frac{1}{2}\norm{z}_2^2 + \langle Lx,z\rangle} p(x) \,\dd z}{Z} \\ 
&= \frac{(2\pi)^{-d/2} \int_{\RR^r} Z(z) e^{-\frac{1}{2}\norm{z}_2^2} p(x;L^\t z) \,\dd z}{Z}
\end{align*}
as claimed.
\end{proof}

The above decomposition implies that $\pstar$ is the marginal distribution of $x$ under the following lifted distribution in $\RR^{d+r}$: 
\[q(x,z) \propto Z(z) e^{-\frac{1}{2}\norm{z}_2^2} p(x;L^\t z).\]
For any fixed $z$, the conditional distribution $q(x\mid{}z)$ is precisely $p(x;L^\t z)$, which we can efficiently sample from using \cref{alg:linear-tilt} from \cref{sec:linear}. Thus, it suffices to (approximately) sample from the marginal distribution over $z$, which is precisely $q(z) \propto Z(z) e^{-\frac{1}{2}\norm{z}_2^2}$. This is where we exploit the low-rank assumption: since $z$ is $r$-dimensional, and $p(x;L^\t z)$ satisfies appropriate smoothness in $z$, it suffices to explicitly estimate the densities $q(z)$ for $z$ on a grid of cardinality (roughly) $\exp(r)$. We accomplish this using $\partitionalg$ (\cref{alg:partition-estimator}), which approximates $Z(z)$ (given the vector $v = L^\t z$) by telescoping Monte Carlo approximations. See \cref{alg:psd-tilt} for the pseudocode.

\begin{remark}
A natural alternative approach to sample from $q(x,z)$ would be Gibbs sampling: i.e. alternately sample from $q(x\mid{}z)$ and $q(z\mid{}x)$, using that both conditional distributions are tractable. However, examples can be constructed in which the Markov chains does not mix rapidly. In particular, if $p$ is a mixture of two Gaussians defined as $p = \frac{1}{2} \mathcal{N}(-u, \sigma^2 I) + \frac{1}{2} \mathcal{N}(u, \sigma^2 I)$, and we tilt with a quadratic reward $\rr(x) = \lambda uu^{\top}$, it can be seen that $q(z|x) = \mathcal{N}(\sqrt{2 \lambda} u^{\top} x, 1)(z)$ and $q(x|z) \propto p(x) \exp(\sqrt{2 \lambda} z u^{\top} x)$. So, if $z$ is positive, then $q(x|z)$ is biased towards the $+u$ mode, and $q(z|x)$ has positive mean $\approx \sqrt{2 \lambda}\|u\|^2$ so it will remain positive. Thus, the tilt induces a metastability of the Gibbs sampler. 
%\ar{Perhaps add a sentence that Gibbs sampling is a natural strategy, in which we alternate b/w $x$ and $z$ but this doesn't work because it can get stuck in a ``mode'' of $p$}
\end{remark}

\iffalse
\begin{algorithm}[p]
\caption{$\basealg$: Sampling from diffusion model via SDE \cite{chen2023sampling}}
\label{alg:base-sampler}
\begin{algorithmic}[1]
  \State \textbf{input:} Score functions $(s_\sigma)_{\sigma \geq 0}$, error tolerance $\epsilon>0$.
  \State \dhruv{\cite{chen2023sampling}}
\end{algorithmic}
\end{algorithm}
\fi

\begin{algorithm}[t]
\caption{$\psdalg$: Sampling from diffusion model with PSD quadratic tilt}
\label{alg:psd-tilt}
\begin{algorithmic}[1]
  \State \textbf{input:} Score functions $(s_\sigma)_{\sigma \in(0,1)}$, reward matrix $L \in \RR^{r\times d}$, norm bounds $D \geq 1$ and $\Cnorm \geq 1$, final error tolerance $\epfinal > 0$.
  \State Set $R := D+2\sqrt{r}+2\sqrt{\log(54/\epfinal)}$, $\gamma := \epfinal/(54D)$.
  \State Define $\cS := \gamma \ZZ^r \cap \CB_{r,2}(R) \subseteq \RR^r$.
  \State Set $N := \Cnorm R\norm{L}_2$, $\epsilon_1 := \epfinal^2/(72\Cnorm)$, $\epsilon_2 := \epfinal/3$, $\delta_1 := \epfinal^2/(72\Cnorm|\cS|)$.
  \For{$z \in \cS$}
    \State $\Zhat(z) \gets \partitionalg((s_\sigma)_\sigma, L^\t z, \epsilon_1, \delta_1, \Cnorm)$. \Comment{\cref{alg:partition-estimator}}
  \EndFor
  \State Set $\Zhat := \sum_{z \in \cS} \Zhat(z) \exp(-\frac{1}{2}\norm{z}_2^2)$ and define $\phat_z(z) := \Zhat(z) \exp(-\frac{1}{2}\norm{z}_2^2)/\Zhat$.
  \State Sample $\ztil \sim \phat_z$.
  \State Sample $\xtil \gets \linearalg((s_\sigma)_\sigma, L^\t \ztil, \epsilon_2, \Cnorm).$ \Comment{\cref{alg:linear-tilt}}
  \State \textbf{return:} $\xtil$.
\end{algorithmic}
\end{algorithm}

\begin{algorithm}[t]
\caption{$\partitionalg$: Estimate normalization for linear tilt}
\label{alg:partition-estimator}
\begin{algorithmic}[1]
  \State \textbf{input:} Score functions $(s_\sigma)_{\sigma \in(0,1)}$, tilt vector $v \in \RR^d$, error tolerance $\epsilon>0$, failure probability $\delta \in (0,1/2)$, norm bound $\Cnorm \geq 1$.
  \State Set $N := \Cnorm\norm{v}_2$, $\epsilon' := \frac{\epsilon}{2(1+e)eN}$, and $M := e^{-2}\log(2N/\delta) / (\epsilon')^2$.
  \For{$1 \leq n \leq N$}
        \For{$1 \leq m \leq M$}
            \State $x\ind{m} \gets \linearalg((s_\sigma)_\sigma, \frac{(n-1)}{N}v, \epsilon',\Cnorm)$. \Comment{\cref{alg:linear-tilt}}
            %\State $y\ind{m} \gets \linearalg((s_\sigma)_\sigma, \frac{n}{N}v, \epsilon')$
        \EndFor
        \State Set $\hat\kappa(n) := \frac{1}{M}\sum_{m=1}^M \exp(\langle x\ind{m}, \frac{1}{N}v\rangle)$.
        %\State Set $\hat\beta := \frac{1}{M}\sum_{m=1}^M \exp(-\langle y\ind{m}, \frac{1}{N}v\rangle\indic[|\langle y\ind{m},v\rangle| \leq N])$.
        %\State Set
        %\[\hat\kappa(n) := \begin{cases} \hat\alpha \text{ if } \hat\alpha \geq 1/2 \\ 
        %\hat\beta^{-1} \text{ otherwise } \end{cases}.\]
  \EndFor
  \State \textbf{return:} $\hat\kappa := \prod_{n=1}^N \hat\kappa(n).$
\end{algorithmic}
\end{algorithm}

In \cref{sec:partitionalg}, we analyze the subroutine $\partitionalg$. In \cref{sec:psdalg} we complete the proof of \cref{theorem:psd}.

\subsection{Analysis of $\partitionalg$}\label{sec:partitionalg}

Given $z \in \RR^r$, a first attempt at estimating $Z(z)$ would be Monte Carlo estimation using samples from $p$. However, this could take exponential time in the problem parameters. Instead, we observe that for any sufficiently close $z,z' \in \RR^r$, it is possible to efficiently estimate $Z(z)/Z(z')$ using samples from $p(\cdot;L^\t z')$---which we can (approximately) obtain using $\linearalg$. Thus, we can estimate $Z(z)$ by telescoping along a path from $Z(0) = 1$. This idea is formalized in \cref{alg:partition-estimator} and analyzed in \cref{lemma:partitionalg-analysis} below. 

\begin{lemma}\label{lemma:partitionalg-analysis}
Let $v \in \RR^d$, $\epsilon,\delta \in (0, 1/2)$, and $N \in \NN$. Suppose that $N \geq \Cnorm\norm{v}_2$. Then the output $\hat\kappa \gets \partitionalg((s_\sigma)_\sigma, v, \epsilon,\delta,\Cnorm)$ satisfies
\begin{equation}\frac{\hat\kappa}{\EE_{x\sim p}[\exp(\langle x,v\rangle)]} \in [1-\epsilon, 1+\epsilon]\label{eq:kappa-acc}\end{equation}
with probability at least $1-\delta$. Moreover, the time complexity of the algorithm is at most \[\poly(d,\epsilon^{-1}, \Cnorm, \norm{v}_2, \log(1/\delta)).\]
\end{lemma}

\begin{proof}
Fix $1 \leq n \leq N$. Set $u := \frac{(n-1)}{N} v$ and $w := \frac{n}{N} v$. Let $\phat(\cdot;u)$ denote the law of \[\linearalg((s_\sigma)_\sigma,u,\epsilon',\Cnorm).\] By \cref{lemma:linearalg-analysis}, we have $\WW(\phat(\cdot;u),p(\cdot;u)) \leq \epsilon'$ and $\supp(\phat(\cdot;u)) \subseteq \CB_{d,2}(\Cnorm)$. %Let $q_u$ denote the law of $X\indic[|\langle X,v\rangle| \leq N]$ under $X \sim \phat(\cdot;u)$.

%By assumption, $\Pr_{X \sim p(\cdot;u)}[\norm{v}_2 + |\langle X,v\rangle| > N] = 0$ since $\supp(p(\cdot;u)) \subseteq \supp(p)$. Therefore $\Pr_{X \sim \phat(\cdot;u)}[|\langle X,v\rangle|> N] \leq \sqrt{\epsilon'}$. It follows from the definition of $q_u$ that $\TV(q_u, \phat(\cdot;u)) \leq \epsilon'$ and hence $\TV(q_u,p(\cdot;u)) \leq 2\epsilon'$.

Define $\xi\ind{m} := \exp(\langle x\ind{m}, \frac{1}{N}v\rangle)$ for $1 \leq m \leq M$. Observe that $\xi\ind{1},\dots,\xi\ind{M}$ are i.i.d. random variables with $\xi\ind{m} \in [0, e]$ almost surely (by assumption on $N$, and the fact that $\phat(\cdot;u)$ is supported on $\CB_{d,2}(\Cnorm)$), and $\EE[\xi\ind{m}] = \EE_{x \sim \phat(\cdot;u)}[\exp(\langle x,\frac{1}{N}v\rangle)]$. It follows from Hoeffding's inequality and choice of $M$ that with probability at least $1-\delta/N$,
\[\left|\hat\kappa(n) - \EE_{x \sim \phat(\cdot;u)}[\exp(\langle x,\frac{1}{N}v\rangle)]\right| \leq \epsilon'.\]
Moreover, since $x \mapsto e^x$ is $e$-Lipschitz in $(-\infty,1]$, we can bound
\begin{align*}
\left|\EE_{x \sim \phat(\cdot;u)}[\exp(\langle x,\frac{1}{N}v\rangle)] - \EE_{x \sim p(\cdot;u)}[\exp(\langle x,\frac{1}{N}v\rangle)]\right|
&\leq \frac{e \WW(\phat(\cdot;u),p(\cdot;u)) \norm{v}_2}{N} \\ 
&\leq e\epsilon'.
\end{align*}
Combining the preceding bounds gives
\begin{align*} 
\left|\hat\kappa(n)- \frac{\EE_{x \sim p}[\exp(\langle x,\frac{n}{N}v\rangle)]}{\EE_{x \sim p}[\exp(\langle x,\frac{n-1}{N}v\rangle)]}\right| 
&= \left|\hat\kappa(n) - \EE_{x \sim p(\cdot;u)}[\exp(\langle x,\frac{1}{N}v\rangle)]\right|  \\ 
%&\leq \epsilon' + e \cdot \TV(q_u,p(\cdot;u))  \\ 
&\leq (1+e)\epsilon'.
\end{align*}
Since
\[\frac{\EE_{x \sim p}[\exp(\langle x,\frac{n}{N}v\rangle)]}{\EE_{x \sim p}[\exp(\langle x,\frac{n-1}{N}v\rangle)]} = \EE_{x \sim p(\cdot;u)}[\exp(\langle x,\frac{1}{N}v\rangle)] \geq 1/e,\]
in the above event we have
\[\left|\frac{\hat\kappa(n) \cdot \EE_{x \sim p}[\exp(\langle x,\frac{n-1}{N}v\rangle)]}{\EE_{x \sim p}[\exp(\langle x,\frac{n}{N}v\rangle)]} - 1\right| \leq (1+e)\epsilon' \cdot \frac{\EE_{x \sim p}[\exp(\langle x,\frac{n-1}{N}v\rangle)]}{\EE_{x \sim p}[\exp(\langle x,\frac{n}{N}v\rangle)]} \leq (1+e)e\epsilon'.\]
By the union bound, we have with probability at least $1-\delta$ that
\begin{align*}
\hat\kappa 
&= \prod_{n=1}^N \hat\kappa(n) \\ 
&\leq \prod_{n=1}^N (1 + (1+e)e\epsilon') \frac{\EE_{x \sim p}[\exp(\langle x,\frac{n}{N}v\rangle)]}{\EE_{x \sim p}[\exp(\langle x,\frac{n-1}{N}v\rangle)]} \\ 
&\leq \exp((1+e)e\epsilon' N) \prod_{n=1}^N \frac{\EE_{x \sim p}[\exp(\langle x,\frac{n}{N}v\rangle)]}{\EE_{x \sim p}[\exp(\langle x,\frac{n-1}{N}v\rangle)]} \\ 
&\leq (1+\epsilon) \EE_{x \sim p}[\exp(\langle x,v\rangle)]
\end{align*}
so long as $\epsilon' \leq \frac{\epsilon}{2(1+e)eN}$, and similarly 
\begin{align*}
\hat\kappa 
&\geq \prod_{n=1}^N (1 - (1+e)e\epsilon') \frac{\EE_{x \sim p}[\exp(\langle x,\frac{n}{N}v\rangle)]}{\EE_{x \sim p}[\exp(\langle x,\frac{n-1}{N}v\rangle)]} \\ 
&\geq \exp(-2(1+e)e\epsilon' N) \prod_{n=1}^N \frac{\EE_{x \sim p}[\exp(\langle x,\frac{n}{N}v\rangle)]}{\EE_{x \sim p}[\exp(\langle x,\frac{n-1}{N}v\rangle)]} \\ 
&\geq (1-\epsilon) \EE_{x \sim p}[\exp(\langle x,v\rangle)],
\end{align*}
which completes the proof of \cref{eq:kappa-acc}. The time complexity of $\partitionalg$ is dominated by $MN$ calls to $\linearalg$. Thus, the claimed time complexity bound follows from \cref{lemma:linearalg-analysis} and choice of $N,M$.
\end{proof}

\subsection{Analysis of $\psdalg$}\label{sec:psdalg}

With the analysis of $\partitionalg$ in hand, the proof of \cref{theorem:psd} is straightforward. There are three types of errors in $\psdalg$ to handle: (1) error from estimation of the normalization constants $Z(z)$, which is bounded using \cref{lemma:partitionalg-analysis}; (2) error in sampling from $p(\cdot;\ztil)$, which is bounded using \cref{lemma:linearalg-analysis}, and (3) error due to discretization of the Hubbard-Stratonovich transform, for which the bound is deferred to \cref{lemma:discretization-error}. Below, we accumulate these errors to complete the analysis of $\psdalg$ and prove \cref{theorem:psd}.

\begin{namedproof}{\cref{theorem:psd}}
For each $y \in \RR^d$ let $\phat(\cdot;y)$ denote the law of $\linearalg((s_\sigma)_\sigma, y, \epsilon_2,\Cnorm)$. 
For the purposes of the analysis, we define distributions $q_1, q_2 \in \Delta(\RR^d)$ by 
\[q_1(x) = \frac{1}{\sum_{z\in\cS} Z(z) e^{-\frac{1}{2}\norm{z}_2^2}} \sum_{z\in\cS} Z(z) e^{-\frac{1}{2}\norm{z}_2^2} \phat(x;L^\t z)\]
and
\[q_2(x) = \frac{1}{\sum_{z\in\cS} Z(z) e^{-\frac{1}{2}\norm{z}_2^2}} \sum_{z\in\cS} Z(z) e^{-\frac{1}{2}\norm{z}_2^2} p(x;L^\t z)\]
We will decompose
\begin{align}
\WW(\mu,\pstar) 
&\leq \WW(\mu,q_1) + \WW(q_1,q_2) + \WW(q_2, \pstar) \nonumber \\ 
&\leq \Cnorm \sqrt{\TV(\mu,q_1)} + \WW(q_1,q_2) + \Cnorm\sqrt{\TV(q_2,\pstar)} \label{eq:w2-decomp}
\end{align}
where the second inequality uses the fact that all of the above distributions are supported on $\CB_{d,2}(\Cnorm)$. We start by bounding $\TV(\mu,q_1)$. Applying \cref{lemma:partitionalg-analysis} and a union bound over $z \in \cS$, we get that in an event $\cE$ with probability at least $1-\delta_1|\cS|$, for all $z \in \cS$, 
\[\frac{\Zhat(z)}{Z(z)} = \frac{\Zhat(z)}{\EE_{x\sim p}[\exp(\langle x,L^\t z\rangle)]} \in [1-\epsilon_1,1+\epsilon_1].\]
Condition on $(\Zhat(z):z\in\cS)$ and suppose that event $\cE$ holds. Let $\nu(\cdot\mid{}\Zhat) \in \Delta(\RR^d)$ denote the conditional law of the output $\xtil$. Then $\nu(\cdot\mid{}\Zhat)$ has density
\[\nu(x\mid{}\Zhat) = \sum_{z \in \cS} \phat_z(z) \phat(x;L^\t z) = \frac{1}{\Zhat} \sum_{z\in\cS} \Zhat(z) e^{-\frac{1}{2}\norm{z}_2^2} \phat(x;L^\t z).\] Define $f(x) = \sum_{z\in\cS} Z(z) e^{-\frac{1}{2}\norm{z}_2^2} \phat(x;L^\t z)$ and $g(x) = \sum_{z\in\cS} \Zhat(z) e^{-\frac{1}{2}\norm{z}_2^2} \phat(x;L^\t z)$. Then
\begin{align*}
\int_x |f(x) - g(x)|\, \dd x 
&= \int_x \sum_{z\in\cS} |Z(z) - \Zhat(z)| e^{-\frac{1}{2}\norm{z}_2^2} \phat(x;L^\t z) \, \dd x \\ 
&\leq \epsilon_1 \int_x \sum_{z\in\cS} Z(z) e^{-\frac{1}{2}\norm{z}_2^2} \phat(x;L^\t z) \, \dd x \\ 
&= \epsilon_1 \int_x f(x) \, \dd x.
\end{align*}
Thus, since $\nu(x\mid{}\Zhat) \propto g(x)$ and $q_1(x) \propto f(x)$, \cref{lemma:tv-unnormalized-densities} implies that $\TV(\nu(\cdot\mid{}\Zhat),q_1) \leq 2\epsilon_1$. Since this bound holds for all $\Zhat\in\cE$, and since $\Pr[\cE] \geq 1-\delta_1|\cS|$, we get \[\TV(\mu,q_1) = \TV(\EE[\nu(\cdot\mid{}\Zhat)],q_1) \leq 2\epsilon_1 + \delta_1|\cS|.\]
Next, we bound $\WW(q_1, q_2)$. By \cref{lemma:linearalg-analysis} it holds that $\WW(\phat(\cdot;L^\t z),p(\cdot;L^\t z)) \leq \epsilon_2$ for all $z \in \RR^r$. Let $p_z \in \Delta(\RR^r)$ be defined by $p_z(z) \propto Z(z) e^{-\frac{1}{2}\norm{z}_2^2}$. Then $q_1$ is the density of the random variable $X$ obtained by sampling $z \sim p_z$ and $X \sim \phat(\cdot;L^\t z)$. Moreover, $q_2$ is the density of the random variable $Y$ obtained by sampling $z \sim p_z$ and $Y \sim p(\cdot;L^\t z)$. Since for any fixed $z$, there is a coupling of $X$ and $Y$ (in the event that $z$ is realized) with $\EE[\norm{X-Y}_2^2 \mid{} z] \leq \epsilon_2^2$, it follows that there is a coupling of $X$ and $Y$ with $\EE[\norm{X-Y}_2^2] \leq \epsilon_2^2$. Thus, we have $\WW(q_1,q_2) \leq \epsilon_2$.

Finally, we apply \cref{lemma:discretization-error} with parameter $\epsilon := \epfinal^2/(162\Cnorm)$. By choice of parameters $R$ and $\gamma$, we get that $\TV(q_2, \pstar) \leq 18\epsilon = \epfinal^2/(9\Cnorm)$. We conclude from \cref{eq:w2-decomp} that
\[\TV(\mu,\pstar) \leq \Cnorm\sqrt{2\epsilon_1 + \delta_1|\cS|} + \epsilon_2 + \Cnorm\sqrt{\epfinal^2/(9\Cnorm)} \leq \epfinal\]
by choice of $\epsilon_1$, $\epsilon_2$, and $\delta_1$. Finally, the time complexity of $\psdalg$ is dominated by $|\cS|$ calls to $\partitionalg$ and one call to $\linearalg$. Note that $|\cS| \leq (2R/\gamma)^r \leq \poly(D^r, r^r, \epfinal^{-r})$. The claimed time complexity bound therefore follows from \cref{lemma:linearalg-analysis,lemma:partitionalg-analysis}.
\end{namedproof}

\section{Conclusion} 

In this paper, we considered the task of sampling from a diffusion model, tilted by a quadratic reward. We provide a fine-grained analysis of the computational tractability of this task through the lens of the rank of the quadratic form. In particular, this task is computationally intractable even for rank-1 negative-definite tilts. For low-rank positive-definite tilts, we give an efficient algorithm based on two conceptually new algorithmic ingredients: sampling from linearly tilted diffusion models, and the Hubbard-Stratonovich transform. 

There are many natural directions for further work. Our results, like many other results in the literature, assume that the score oracle is exact---whereas in practice, the oracle we have access to is trained, and thus would have errors. Handling errors is non-trivial because they will be small on average under the base distribution---but algorithms which sample from the tilted distribution may deviate substantially from the trajectories that would arise when sampling from the base distribution. 

Towards handling more complex rewards, it would also be interesting to understand what interaction between the structure of the reward and the structure of the base distribution allow for efficient algorithms. 

\arxiv{
\paragraph{Acknowledgments.} AM is supported in part by a Microsoft Trustworthy AI Grant, NSF award CCF-2430381, ONR grant N00014-22-1-2339, and a David and Lucile Packard Fellowship. AR is supported in part by NSF awards IIS-2211907, CCF-2238523, IIS-2403275, an Amazon Research
Award, ONR award N000142512124, a Google Research Scholar Award, and an OpenAI Superalignment Fast Grant. DR is supported by NSF awards CCF-2430381 and DMS-2022448, and ONR grant N00014-22-1-2339.
}

% Acknowledgments---Will not appear in anonymized version
\colt{
\acks{We thank a bunch of people and funding agency.}
}

\arxiv{
\bibliographystyle{alpha}
}
\bibliography{bib}

\appendix
\crefalias{section}{appendix}

\section{Steering with general positive-definite rewards is hard}\label{sec:psd-high-rank}

In this section we prove the following hardness result, which shows that the low-rank assumption in \cref{theorem:psd} cannot be removed.

\begin{theorem}\label{thm:psd-general-hardness}
Suppose that there is a randomized algorithm $\MA$ with the following property. For any integer $d \in \NN$, any distribution $p \in \Delta(\RR^d)$ satisfying \cref{ass:score} with parameter $\Cnorm \geq 1$, and any positive semi-definite $A \in \RR^{d\times d}$, the output $\xtil \in \RR^d$ of $\MA((\grad \log p_\sigma)_{\sigma\in(0,1)}, A, \Cnorm)$ has law $\nu$ satisfying $\WW(\nu, \pstar) \leq 1/4$, where $\pstar \in \Delta(\RR^d)$ is the tilted distribution \[\pstar(x) \propto p(x) \exp(x^\t A x).\] Moreover, the time complexity of $\MA((\grad \log p_\sigma)_{\sigma\in(0,1)}, w, \Cnorm)$ is $\poly(d,\Cnorm,\norm{A}_2)$. 

Then, $\NP \subseteq \BPP$. 
\end{theorem}

We prove \cref{thm:psd-general-hardness} by reduction from the $\NP$-hard problem MAX-CUT:

\begin{definition}[MAX-CUT]
Given a graph $G = ([d],E)$ and an integer $k$, the MAX-CUT problem is to determine whether there is a set $S \subseteq [d]$ such that $c_G(S) \geq k$, where
\[c_G(S) := |\{(u,v) \in E: (u \in S \text{ and } v \not \in S) \text{ or } (u \not\in S \text{ and } v \in S) \}|.\]
\end{definition}

The proof is analogous to that of \cref{thm:nsd-hardness}: given an instance of MAX-CUT, we construct a tilting problem so that $\pstar$ puts most of its mass on solutions to the instance.

\begin{namedproof}{\cref{thm:psd-general-hardness}}
We give an algorithm for MAX-CUT using $\MA$ as a subroutine. Fix an instance of MAX-CUT, which is described by a graph $G = ([d],E)$ and integer $k$. Define $p := \Unif(\{0,1\}^d)$ and $\beta := d + 100$ and $\Cnorm=\sqrt{d}$ and
\[A := \beta \left(\sum_{(u,v) \in E} (e_{uu} + e_{vv} - e_{uv} - e_{vu})\right) \in \RR^{d \times d}\]
where $e_{ij}$ is the $d \times d$ matrix with a $1$ in entry $(i,j)$ and $0$ everywhere else. We sample $\xtil \in \RR^d$ from $\MA((s_\sigma)_{\sigma\in(0,1)},A,\Cnorm)$, where $s_\sigma = \grad \log p_\sigma$, using the fact that $p_\sigma$ is a product distribution where each marginal is a mixture of two Gaussians, and hence the score $\grad \log p_\sigma(x)$ can be explicitly evaluated for any $x$. We then compute $\hat{x}\in\{0,1\}^d$ defined by $\hat{x}_i := \indic[\xtil_i \geq 1/2]$ for each $i \in [d]$. We return YES if $\beta^{-1} \hat{x}^\t A \hat{x} \geq k$ and NO otherwise.

\paragraph{Analysis.} Observe that for any $x \in \RR^d$,
\[\beta^{-1} x^\t A x = \sum_{(u,v) \in E} (x_u^2 + x_v^2 - 2x_u x_v) = \sum_{(u,v) \in E} (x_u - x_v)^2 \geq 0\]
and thus $A$ is positive semi-definite. Moreover, if $x = \indic[S]$ for some $S \subset [d]$, then we have
\[\beta^{-1} x^\t A x = c_G(S).\]
Define $\pstar \in \Delta(\RR^d)$ by
\[\pstar(x) \propto p(x) \exp(x^\t A x).\]
Let $U \subset \{-1,1\}^d$ be the set of $x$ such that $\beta^{-1} x^\t A x = \max_{y \in \{-1,1\}^d} \beta^{-1} y^\t A y =: c_G$. Then \[\sum_{x\in U} \exp(x^\t A x) \geq \exp(\beta c_G).\] Moreover, for any $x \in \{-1,1\}^d \setminus U$, we have $\beta^{-1} x^\t A x \leq c_G - 1$, so \[\sum_{x \in \{-1,1\}^d \setminus U} \exp(x^\t A x) \leq 2^d \exp(\beta(c_G - 1)).\]
It follows that
\begin{align*} 
\pstar(\{-1,1\}^d \setminus U) 
&= \frac{\sum_{x\in \{-1,1\}^d \setminus U} \exp(x^\t A x)}{\sum_{x\in \{-1,1\}^d} \exp(x^\t A x)} \\ 
&\leq \frac{2^d \exp(\beta(c_G - 1))}{\exp(\beta c_G) + 2^d \exp(\beta (c_G - 1))}  \\ 
&= \frac{2^d}{e^\beta + 2^d} \\ 
&\leq 1/100
\end{align*}
by choice of $\beta$. Let $\nu$ be the law of the sample $\xtil$. By assumption, we have $\WW(\nu,\pstar) \leq 1/4$. Let $V := \{y \in \RR^d: \indic[y \geq 1/2] \in U\}$ where $\indic[y\geq 1/2]$ refers to the coordinate-wise thresholding of $y$. By \cref{lem:rounding-mass} and the preceding bounds, we have $\nu(V) \geq \pstar(U) - (1/4)^2 \geq 0.9$. Thus, the rounded vector $\hat{x}$ satisfies $\beta^{-1} \hat{x}^\t A \hat{x} = c_G$ with probability at least $0.9$. If the answer to the MAX-CUT instance is NO, then the output is always NO, since it holds almost surely that $\beta^{-1} \hat{x}^\t A \hat{x} = c_G(S) \leq c_G < k$ where $\hat{x} = \indic[S]$. If the answer is YES, then with probability at least $0.9$ we have $\beta^{-1} \hat{x}^\t A \hat{x} = c_G \geq k$ and hence the output is YES. Moreover, the time complexity of the reduction is $\poly(d)$ by assumption. Thus, $\NP \subseteq \BPP$.
\end{namedproof}

\section{Omitted details from \cref{sec:nsd}}

The following lemma lower bounds the Wasserstein distance between a distribution on the hypercube $\mu$ and any other distribution $\nu$ in terms of the measure that $\nu$ puts on the ``rounding'' of the measure onto the hypercube:
\begin{lemma}
\label{lem:rounding-mass}
Let $\mu$ be any probability measure supported on $\{\pm1\}^d$, let $S\subseteq \{\pm1\}^d$,
and let $\nu$ be any probability measure on $\RR^d$ with finite second moment. Let $R_{S} := \{y\in\RR^d : \sgn(y)\in S\}.$ Then:
\[
W_2(\mu,\nu)^2 \;\ge\; \mu(S) - \nu(R_S).
\]
Thus, if $W_2(\mu,\nu)\le \varepsilon$, then
\[
\nu(R_S) \;\ge\; \mu(S) - \varepsilon^2.
\]
\end{lemma}
\begin{proof}
Fix an arbitrary coupling $(X,Y)$ of $\mu$ and $\nu$, i.e.\ $X\sim \mu$ and $Y\sim \nu$.
Define the event
\[
E := \{ X\in S \ \text{and}\ Y\notin R_S \}.
\]
On $E$, we have $X\in S \subseteq \{\pm1\}^n$, while $Y\notin R_S$ means $\sgn(Y)\notin S$.
In particular, since $X\in S$ but $\sgn(Y)\notin S$, we must have $\sgn(Y)\neq X$.
Therefore, on the event $E$ we have $\|X-Y\|_2^2 \ge 1$.
Consequently,
\[
\E\|X-Y\|_2^2 \;\ge\; \E\big[\|X-Y\|_2^2 \cdot \mathbf{1}_E\big] \;\ge\; \E[\mathbf{1}_E] \;=\; \Pp(E).\]

Next we lower bound $\Pp(E)$.
Since $\Pp(E)=\Pp(X\in S) - \Pp(X\in S,\ Y\in R_S)$, and since $\Pp(X\in S)=\mu(S)$ and $\Pp(X\in S,\ Y\in R_S) \le \Pp(Y\in R_S)=\nu(R_S)$, we get:
\[
\Pp(E) \;\ge\; \mu(S) - \nu(R_S).
\]
Combining the two inequalities, we get that for an arbitrary coupling $(X,Y)$,
\[
\E\|X-Y\|_2^2 \;\ge\; \mu(S) - \nu(R_S).
\]
Taking the infimum over all couplings gives
\[
W_2(\mu,\nu)^2 = \inf \E\|X-Y\|_2^2 \;\ge\; \mu(S) - \nu(R_S),
\]
which is what we wanted. 
\end{proof}

\subsection{Proof of \cref{lem:mass-on-solutions}}\label{sec:mass-on-solutions}

\begin{lemma}[Restatement of \cref{lem:mass-on-solutions}]
Given a PARTITION instance $w\in\mathbb{Z}^d$, define the set $S_w := \{x\in\{\pm1\}^d : w^\top x=0\}$. Assume $S_w\neq \emptyset$ (i.e.\ the PARTITION instance is a YES instance). Then $q_w(S_w) \;\ge\; \frac{200}{201}$, where $q_w \in \Delta(\RR^d)$ is defined by $q_w(x) \propto \exp(-(d+5)\langle x,w\rangle^2)$.
\end{lemma}

\begin{proof}

Let $\beta:= d+5$. Let us also denote
\[
\widetilde Z_w := \sum_{x\in\{\pm1\}^d} \exp\!\big(-\beta (w^\top x)^2\big),
\]
so that
\[
q_w(x) \;=\; \frac{\exp\!\big(-\beta (w^\top x)^2\big)}{\widetilde Z_w}.
\]

For $x\in S_w$, we have $w^\top x=0$, hence $\exp(-\beta (w^\top x)^2)=1$. Thus
\[
\sum_{x\in S_w} \exp\!\big(-\beta(w^\top x)^2\big) = |S_w|.
\]

For $x\notin S_w$, we have $w^\top x \in \mathbb{Z}\setminus\{0\}$, so $|w^\top x|\ge 1$ and therefore
\[
\exp\!\big(-\beta(w^\top x)^2\big) \;\le\; e^{-\beta}.
\]
Hence
\[
\widetilde Z_w
= \sum_{x\in S_w} 1 \;+\; \sum_{x\in \{\pm 1\}^d\setminus S_w}\exp\!\big(-\beta(w^\top x)^2\big)
\;\le\; |S_w| + (2^d-|S_w|)\,e^{-\beta}
\;\le\; |S_w| + 2^d e^{-\beta}.
\]
Therefore
\[
q_w(S_w) \;=\; \frac{|S_w|}{\widetilde Z_w} \;\ge\; \frac{|S_w|}{|S_w| + 2^d e^{-\beta}}.
\]

We lower bound $|S_w|$ and upper bound $2^d e^{-\beta}$.
Since $S_w\neq\emptyset$, pick some $x^\star \in S_w$. Then also $-x^\star \in S_w$ because
$w^\top(-x^\star)=-w^\top x^\star =0$. For $d\ge 2$, $x^\star\neq -x^\star$, hence $
|S_w|\ge 2.$ Next, with $\beta=d+5$,
we have $
2^d e^{-\beta} = 2^d e^{-(d+5)} = e^{-5}\left(\frac{2}{e}\right)^d \le e^{-5} < \frac{1}{100}.$
Combining the bounds,
\[
q_w(S_w) \;\ge\; \frac{|S_w|}{|S_w| + 2^d e^{-\beta}}
\;\ge\; \frac{2}{2 + 1/100}
\;=\; \frac{200}{201}
\]
as claimed.
\end{proof}

\section{Omitted details from \cref{sec:psd}}

We recall notation from \cref{sec:psd}. Fix $d,r \in \NN$. Suppose that $p \in \Delta(\RR^d)$ satisfies \cref{ass:score} with parameter $\Cnorm \geq 1$. Fix $L \in \RR^{r \times d}$, and define $\pstar \in \Delta(\RR^d)$ by
\[\pstar(x)\propto p(x)e^{\frac{1}{2}\norm{Lx}_2^2}.\]

\subsection{Analysis of discretization error}

The Hubbard-Stratonovich transform (\cref{lemma:hs-transform}) decomposes $\pstar$ as a mixture of linear-tilted distributions $p(x;L^\t z)$ (\cref{def:linear-tilt}) where $z$ ranges continuously over $\RR^d$. The following lemma shows that the range of $z$ can be discretized with small error:

\begin{lemma}\label{lemma:discretization-error}
Fix $\epsilon \in (0,1)$ and $ \gamma,R>0$. Let $D := \sup_{x\in\supp(p)}\norm{Lx}_2$. Let $\cS := \gamma \ZZ^r \cap \CB_{r,2}(R)$. Define
\[q(x) \propto \sum_{z \in \cS} Z(z) e^{-\frac{1}{2}\norm{z}_2^2} p(x;L^\t z).\]
If $R \geq D+2\sqrt{r}+2\sqrt{\log(1/\epsilon)}$ and $\gamma \leq \epsilon/D$, then
\[\TV(q, \pstar) \leq 18\epsilon.\]
\end{lemma}

\begin{proof}
For any $z,z' \in \RR^r$, we have 
\[e^{-D\norm{z-z'}_2} \leq \frac{Z(z)}{Z(z')} \leq e^{D\norm{z-z'}_2}\]
and thus, for all $x \in \RR^d$,
\[e^{-2D\norm{z-z'}_2} \leq \frac{p(x;L^\t z)}{p(x;L^\t z')} \leq e^{2D\norm{z-z'}_2}.\]
For each $z \in \cS$ define $B(z) \subset \RR^r$ by $B(z) := z + \CB_{r,\infty}(\gamma/2)$. Then the sets $(B(z): z \in \gamma \ZZ^r)$ partition $\RR^d$, so by \cref{lemma:hs-transform},
\begin{align*}
\pstar(x) 
&\propto \int_{\RR^r} Z(z) e^{-\frac{1}{2}\norm{z}_2^2} p(x;L^\t z) \, \dd z \\ 
&= \sum_{z\in \gamma\ZZ^r} \int_{B(z)} Z(z') e^{-\frac{1}{2}\norm{z'}_2^2} p(x;L^\t z') \, \dd z'.
\end{align*}
For convenience, define
\begin{align*}
\pstar_R(x) 
&\propto f(x) := \sum_{z\in \cS} \int_{B(z)} Z(z') e^{-\frac{1}{2}\norm{z'}_2^2} p(x;L^\t z') \, \dd z'.
\end{align*}
We observe that for any $x \in \supp(p)$,
\begin{align*}
1 
&\geq \frac{\sum_{z \in \cS} \int_{B(z)} e^{-\frac{1}{2}\norm{z'}_2^2 + \langle Lx,z'\rangle} \, \dd z'}{\int_{\RR^r} e^{-\frac{1}{2}\norm{z}_2^2 + \langle Lx,z\rangle} \, \dd z} \\ 
&\geq \frac{\int_{\CB_{r,2}(R)} e^{-\frac{1}{2}\norm{z}_2^2 + \langle Lx,z\rangle} \, \dd z}{\int_{\RR^r} e^{-\frac{1}{2}\norm{z}_2^2 + \langle Lx,z\rangle} \, \dd z} \\ 
&= \Pr_{z \sim \normal(Lx, I_r)}[\norm{z}_2 \leq R] \\ 
&\geq 1-\epsilon
\end{align*}
by \cref{lemma:chi-squared-concentration}, the assumption that $R \geq D + 2\sqrt{r}+2\sqrt{\log(1/\epsilon)}$, and the fact that $\norm{Lx}_2 \leq D$. It follows that
\begin{align*}
(1-\epsilon)\int_{\RR^d} p(x) \int_{\RR^r} e^{-\frac{1}{2}\norm{z'}_2^2 + \langle Lx,z'\rangle} \, \dd z' \, \dd x
&\leq \int_{\RR^d} p(x) \sum_{z \in \cS} \int_{B(z)} e^{-\frac{1}{2}\norm{z'}_2^2 + \langle Lx,z'\rangle} \, \dd z' \, \dd x  \\ 
&\leq \int_{\RR^d} p(x) \int_{\RR^r} e^{-\frac{1}{2}\norm{z'}_2^2 + \langle Lx,z'\rangle} \, \dd z' \, \dd x
\end{align*}
and hence
\begin{align*}
\pstar_R(x) = \frac{p(x) \sum_{z \in \cS} \int_{B(z)} e^{-\frac{1}{2}\norm{z'}_2^2 + \langle Lx,z'\rangle} \, \dd z'}{\int_{\RR^d} p(x) \sum_{z \in \cS} \int_{B(z)} e^{-\frac{1}{2}\norm{z'}_2^2 + \langle Lx,z'\rangle} \, \dd z' \, \dd x} \in [(1-\epsilon)\pstar(x), (1-\epsilon)^{-1} \pstar(x)] 
\end{align*}
which means that \[\TV(\pstar,\pstar_R) = \int_{\RR^d} |\pstar(x) - \pstar_R(x)| \, \dd x \leq 2\epsilon \int_{\RR^d} \pstar(x) \, \dd x = 2\epsilon.\] Next, we compare $\pstar_R$ with $q$. Write
\[g(x) := \sum_{z\in \cS} Z(z) e^{-\frac{1}{2}\norm{z}_2^2} p(x;L^\t z).\]
Observe that for any $z,z' \in \CB_{r,2}(R)$ with $z' \in B(z)$, we have $\norm{z-z'}_2 \leq \gamma \leq \epsilon/R$, and therefore $|Z(z)-Z(z')| \leq 2\epsilon Z(z)$ and $|p(x;L^\t z)-p(x;L^\t z')| \leq 4\epsilon p(x;L^\t z)$ and $|e^{-\frac{1}{2}\norm{z}_2^2} - e^{-\frac{1}{2} \norm{z'}_2^2}| \leq 2\epsilon \cdot e^{-\frac{1}{2}\norm{z}_2^2}$. Thus,
\begin{align*}
\int_{\RR^d} |f(x)-g(x)| \,\dd x 
&\leq \sum_{z\in \cS} \int_{B(z)} \int_{\RR^d} \left|Z(z')e^{-\frac{1}{2}\norm{z'}_2^2} p(x;L^\t z') - Z(z) e^{-\frac{1}{2}\norm{z}_2^2} p(x;L^\t z)\right| \, \dd x \, \dd z' \\ 
&\leq \sum_{z\in \cS} \int_{B(z)} \int_{\RR^d} \left|Z(z') - Z(z)\right| e^{-\frac{1}{2}\norm{z}_2^2} p(x;L^\t z) \, \dd x \, \dd z' \\ 
&+ \sum_{z\in \cS} \int_{B(z)} \int_{\RR^d} Z(z')\left|e^{-\frac{1}{2}\norm{z'}_2^2} - e^{-\frac{1}{2}\norm{z}_2^2}\right| p(x;L^\t z) \, \dd x \, \dd z' \\ 
&+ \sum_{z\in \cS} \int_{B(z)} \int_{\RR^d} Z(z')e^{-\frac{1}{2}\norm{z'}_2^2} \left|p(x;z
) - p(x;L^\t z')\right| \, \dd x \, \dd z' \\ 
&\leq 8\epsilon \cdot \sum_{z\in S} \int_{B(z)} \int_{\RR^d} Z(z')e^{-\frac{1}{2}\norm{z'}_2^2} p(x;z'
) \, \dd x \, \dd z' \\ 
&= 8\epsilon \int_{\RR^d} f(x) \, \dd x.
\end{align*}
It follows from \cref{lemma:tv-unnormalized-densities} that $\TV(\pstar_R, q) \leq 16\epsilon$. Combining the above bounds gives $\TV(\pstar,q) \leq 18\epsilon$ as claimed. 
\end{proof}

\subsection{Technical lemmas}

\begin{lemma}[Concentration of $\chi^2$-random variable]\label{lemma:chi-squared-concentration}
Fix $d \in \NN$ and let $Z \sim \normal(0,I_d)$. Then for any $\epsilon > 0$,
\[\Pr[\norm{Z}_2 > 2\sqrt{d} + 2\sqrt{\log(1/\epsilon)}] \leq \epsilon.\]
\end{lemma}

\begin{lemma}\label{lemma:tv-unnormalized-densities}
Let $f,g:\RR^d \to \RR_{\geq 0}$ be integrable, and let $p,q\in\Delta(\RR^d)$ be defined by $p(x) \propto f(x)$ and $q(x) \propto g(x)$. Then
\[\TV(p,q) \leq \frac{2 \int |f(x)-g(x)| \, \dd x}{\int f(x) \, \dd x}.\]
\end{lemma}

\begin{proof}
Set $Z_f := \int f(x) \, \dd x$ and $Z_g := \int g(x) \, \dd x$. Then we have
\begin{align*}
\TV(p,q) 
&= \int \left| \frac{f(x)}{Z_f} - \frac{g(x)}{Z_g}\right| \, \dd x  \\ 
&\leq \int \left| \frac{f(x)}{Z_f} - \frac{g(x)}{Z_f}\right| \, \dd x + \int \left| \frac{g(x)}{Z_f} - \frac{g(x)}{Z_g}\right| \, \dd x  \\ 
&\leq \frac{1}{Z_f} \int |f(x)-g(x)| \, \dd x + \left|\frac{1}{Z_f} - \frac{1}{Z_g}\right| \int g(x) \, \dd x \\ 
&\leq \frac{1}{Z_f} \int |f(x)-g(x)| \, \dd x + \frac{|Z_f - Z_g|}{Z_f} \\ 
&\leq \frac{1}{Z_f} \int |f(x)-g(x)| \, \dd x + \frac{\int |f(x) - g(x)| \, \dd x}{Z_f} \\ 
&= \frac{2}{Z_f} \int |f(x) - g(x)|\, \dd x
\end{align*}
as claimed.
\end{proof}

\iffalse
\begin{lemma}\label{lemma:kl-tilt-bound}
Let $p,\pstar,q \in \Delta(\RR^d)$ and let $C := \norm{\pstar/p}_\infty$. It holds that
\[\Dkl{\pstar}{q} \leq (1 + C + C\log C)(\Dkl{p}{q} + 1).\]
\end{lemma}

\begin{proof}
Define $f: \RR_{>0} \to \RR$ by $f(x) := x\log(x)$. Then for any $x>0$ and $C \geq 1$, we have
\begin{align*}
f(Cx)
&= C x\log(C x) \\ 
&= C f(x) + x \cdot C \log(C ) \\ 
&\leq C f(x) + (x\log x + 1) \cdot C \log(C ) \\ 
&= (C  + C \log C ) f(x) + C \log C.
\end{align*}
Therefore for any $\alpha \in [1,C]$, by convexity of $f$ and the fact that $f(x) \geq -1$ for all $x \in \RR$,
\begin{align*}
f(\alpha x)
&\leq \frac{\alpha - 1}{C-1} f(Cx) + \frac{C-\alpha}{C-1} f(x) \\ 
&\leq f(Cx) + f(x) + 2 \\ 
&\leq (1 + C + C\log C) f(x) + C \log(C) + 2 \\ 
&\leq (1 + C + C\log C)(f(x) + 1).
\end{align*}
For any $\alpha \in (0,1]$, we have \[f(\alpha x) \leq f(x) + 1 \leq (1+C+C\log(C))(f(x) + 1)\]
as well. Thus,
\begin{align*}
\Dkl{\pstar}{q} 
&= \EE_{x\sim q} f\left(\frac{\pstar(x)}{q(x)}\right) \\ 
&= \EE_{x\sim q} f\left(\frac{\pstar(x)}{p(x)} \cdot \frac{p(x)}{q(x)}\right) \\ 
&\leq (1+C+C\log C) \EE_{x\sim q} \left(f\left(\frac{p(x)}{q(x)}\right) + 1\right) \\ 
&= (1+C+C\log C)(\Dkl{p}{q} + 1)
\end{align*}
as claimed.
\end{proof}
\fi

\end{document}